\documentclass[12pt]{article}
\usepackage[utf8]{inputenc} 
\usepackage[T1]{fontenc}    
\usepackage[
    backref,
    colorlinks,
    linkcolor={purple},
    citecolor={blue},
    urlcolor={red},
    bookmarks=true,
    linktoc=page
    ]{hyperref}
\usepackage{footnotebackref}
\usepackage{url}            

\usepackage{booktabs}       
\usepackage{amsfonts}       
\usepackage{nicefrac}       
\usepackage{microtype}      
\usepackage{xcolor}         
\usepackage{enumitem}
\usepackage[nolist,nohyperlinks]{acronym}
\usepackage{fullpage}
\usepackage{setspace}
\usepackage[ruled,noend]{algorithm2e}
\SetCommentSty{mycommfont}

\usepackage{amsmath,amssymb,amsfonts,mathtools,physics}
\usepackage{latexsym,amsthm}
\usepackage{thm-restate,thmtools}
\usepackage[capitalise,nameinlink]{cleveref}   

\newtheorem*{theorem*}{Theorem (Informal)}
\newtheorem{theorem}{Theorem}
\newtheorem{lem}{Lemma}

\theoremstyle{definition}
\newtheorem{defn}{Definition}
\newcounter{alphcounter}

\newtheorem{assumption}[alphcounter]{Assumption}

\newtheorem{game}{Game}

\theoremstyle{remark}

\theoremstyle{remark}
\newtheorem{remark}{Remark}

\newcommand{\expect}[1]{\mathbb{E}\!\left[#1\right]}
\newcommand{\expectover}[2]{\underset{#2}{\mathbb{E}}\!\left[#1\right]}

\newcommand{\from}{:}
\newcommand{\defeq}{:=}
\newcommand{\dtv}[2]{d_{\mathrm{TV}}\left(#1,#2\right)}

\newcommand{\gen}{\mathrm{gen}}
\newcommand{\genbar}{\overline{\mathrm{gen}}}
\newcommand{\reg}[2]{\mathrm{regret}_{#1,#2}}

\newcommand{\cP}{\mathcal{P}}
\newcommand{\tp}{\text{P}}
\newcommand{\cD}{\mathcal{D}}
\newcommand{\ex}{\mathbb{E}}
\newcommand{\cH}{\mathcal{H}}
\newcommand{\cZ}{\mathcal{Z}}
\newcommand{\nH}[1]{{\Vert{#1}\Vert}_{\cH}}
\newcommand{\nZ}[1]{{\Vert{#1}\Vert}_{\mathcal{Z}}}
\newcommand{\cX}{\mathcal{X}}
\newcommand{\tq}{\text{Q}}
\newcommand{\otb}{\ensuremath{\text{Online-to-Batch}~}}
\newcommand{\learner}{\mathcal{L}}
\newcommand{\reals}{\mathbb{R}}
\newcommand{\regret}{\text{regret}}
\newcommand{\rz}{R_\mathcal{Z}}
\newcommand{\rh}{R_\mathcal{H}}
\newcommand{\gz}{G_\mathcal{Z}}
\newcommand{\gh}{G_\mathcal{H}}
\newcommand{\filter}{\mathcal{F}}

\newcommand{\deltaH}{\Delta_{\cH}}
\newcommand{\expecOver}[1]{\mathbb{E}_{#1}}
\newcommand{\dkl}[2]{{\mathrm{D}}_{\mathrm{KL}}\left(#1\,\Vert\,#2\right)}
\newcommand{\inner}[2]{\langle#1, #2\rangle}
\begin{acronym}
\acro{OTB}{Online-to-Batch}
\acro{EWA}{Exponentially Weighted Averages}
\end{acronym}

\usepackage[square,numbers]{natbib}
\hypersetup{urlcolor=blue, colorlinks=true}
\usepackage{backref}

\usepackage{authblk}

\title{Generalization Bounds for Dependent Data using \\ Online-to-Batch Conversion}

    \author[1]{Sagnik Chatterjee \thanks{sagnikc@iiitd.ac.in}}
    \author[1]{Manuj Mukherjee \thanks{manuj@iiitd.ac.in}}
    \author[1]{Alhad Sethi \thanks{alhad21445@iiitd.ac.in}}
    \affil[1]{Indraprastha Institute of Information Technology, Delhi (IIIT-Delhi)}
  
\date{}
\begin{document}

\maketitle
\vspace{-1.5cm}
\begin{abstract}
  \noindent In this work, we upper bound the generalization error of batch learning algorithms trained on samples drawn from a mixing stochastic process (i.e., a dependent data source) both in expectation and with high probability. Unlike previous results by Mohri et al.\ (2010) and Fu et al.\ (2023), our work does not require any stability assumptions on the batch learner, which allows us to derive upper bounds for any batch learning algorithm trained on dependent data. This is made possible due to our use of the \acf{OTB} conversion framework, which allows us to shift the burden of stability from the batch learner to an artificially constructed online learner. We show that our bounds are equal to the bounds in the i.i.d. setting up to a term that depends on the decay rate of the underlying mixing stochastic process. 
  Central to our analysis is a new notion of algorithmic stability for online learning algorithms based on Wasserstein distances of order one. Furthermore, we prove that the EWA algorithm, a textbook family of online learning algorithms, satisfies our new notion of stability. Following this, we instantiate our bounds using the EWA algorithm.
\end{abstract}

\section{Introduction}
An offline statistical learning algorithm (also referred to as a \textit{batch learner}) $A$ is a randomized map from a set $\mathcal{S}$ of instances drawn from a fixed instance space $\cZ$ to a fixed hypothesis space $\cH$ associated with $A$. In the supervised learning setup, the objective is to \textit{learn} a hypothesis function $h\in\cH$ from a given set $\mathcal{S}$ of \textit{training} instances using $A$, such that $h$ performs `well' on unseen \textit{test} instances with respect to some appropriately chosen loss function. To quantify the performance of the offline learner $A$, a commonly used metric is the \textit{generalization error} of $A$. The generalization error of $A$ is the difference of the average loss incurred by the hypothesis returned by $A$ on training and test instances.

Classically, the generalization error of an offline learner $A$ has been characterized in terms of combinatorial complexity measures of the hypothesis space $\cH$, such as the VC dimension~\citep{yu1994rates,meir2000nonparametric}. With the advent of modern over-parameterized models, where the number of tunable hyper-parameters far exceeds the size of the training set, generalization bounds based on combinatorial measures are often vacuous in nature~\citep{zhang2017,zhang21understanding}. In an effort to come up with better measures of the generalization error of offline learners, researchers have proposed algorithm-dependent generalization error bounds, such as bounds due to algorithmic stability~\citep{bousquet2002}, information-theoretic properties~\citep{russo16,xu_NIPS2017}, or bounds that are PAC-Bayesian in nature~\citep{hellström2024,alquier2024}.
Of particular relevance to this work is the recent work of  \citet{Lugosi2023OnlinetoPACCG}, which applies the \acf{OTB} framework of \citet{cesa2004generalization} to derive algorithm-dependent generalization error bounds for offline learners. 

We remark at this point that most of the generalization error bounds in the literature operate on the underlying assumption that the training and test samples are drawn i.i.d. from the same (unknown) underlying distribution. In many real-world applications, such as learning from time-series dependent data or Federated Learning setups, the i.i.d. assumption does not hold~\citep{vidyasagar2013learning,amiri22fdc,xiong22fdc,iyer2024reviewfdc,li24fdc}.
Prior works that have tried to address generalization error bounds in non-i.i.d. settings require restrictive stability assumptions on the offline learner~\citep{mohri10stab,fu2023sharper}. 

\subsection{Main Results}
In an effort to address the previous shortcomings, we derive bounds on the generalization error of \textit{any} offline learner by extending the aforementioned \otb paradigm to the non-i.i.d. setting, where we assume that the offline learners are trained on data sampled from a stochastic process that is ``mixing" (see \cref{def:mixing}) to a stationary distribution.\footnote{Mixing is a fairly natural assumption that captures the notion of non-i.i.d'ness for learning tasks. See for example~\citep{yu1994rates,meir2000nonparametric,lozano06,mohri07,mohri10stab,Agarwal2011TheGA,duchi12ergodic,kuznetsov2017generalization,zhang19,fu2023sharper}.} The \otb technique allows one to bound the generalization error of an offline learner via the \emph{regret}\footnote{Regret of an online learner is defined as the difference between the total cost incurred by the online learner and any fixed offline learner. See \cref{eq:regret} for a precise definition.} of an artificially constructed online learner. This allows us to derive generalization error bounds for any offline learner $A$, without requiring stability assumptions on $A$ itself (unlike in \citet{mohri10stab,fu2023sharper}), and instead, the stability assumption gets shifted to the artificial online learner. This indicates that our contributions go beyond simply integrating different paradigms. In fact, to the best of our knowledge, along with the concurrent work of \citet{neu2024delayed}, we are the first to provide generalization error bounds for offline learners (possibly not conforming to any notion of stability) trained on non-i.i.d. data.

However, there is an apparent issue: As mentioned previously, our setup requires the online learners in the \otb technique to satisfy a novel notion of stability, which we call \emph{Wasserstein-stability}.\footnote{See \cref{assmp:stability} for a formal treatment.} It is not apriori evident if such classes of online learners even exist. We resolve the above issue by proving that the class of \acf{EWA} learners conforms to the notion of Wasserstein stability, a result that may be of independent interest. This result then allows us to instantiate our generalization error bounds via EWA learners. 
We summarize our results in the following informal theorems. The first result provides a framework to obtain generalization error bounds on offline learners via regret of Wasserstein stable online learners.

\begin{theorem*}
    Let $A$ be any offline learner trained on a set of $n$ instances drawn from a suitably mixing random process. Then, 
    the expected generalization error of $A$ is upper bounded by
    $\frac{1}{n}\expect{\regret_{\learner}} + O\left(\frac{1}{n}\right)$, where $\learner$ is any arbitrary Wasserstein-stable online learner. Furthermore, the generalization error of $A$ is upper bounded by 
    $
    {\frac{1}{n}\regret_{\learner} + O\left(\sqrt{\frac{1}{n}\cdot\log{(\nicefrac{1}{\delta})}}\right)}
    $
    {w.p. $\geq 1-\delta$, for any $\delta>0$.}
\end{theorem*}

The next result instantiates the framework in the previous theorem by using the EWA online learner. 
\begin{theorem*}
    For any distribution $P_1$ over the hypothesis space $\cH$, any $\delta>0$, and an appropriately chosen constant $C>0$, w.p. $\geq 1-\delta$, the generalization error of any offline learning algorithm $A$ trained on $S_n=\left(Z_1,\ldots,Z_n\right)$ drawn from a suitably mixing process ${\{Z_t\}}_{t\in\mathbb{N}}$ is
    \begin{align*}
        O\left(\frac{\dkl{P_{A(S_n)}}{P_1}+C\log{n}+\sqrt{\log{n}\log{(1/\delta)}}}{\sqrt{n}}\right), 
    \end{align*}
    where $P_{A(S_n)}$ is the distribution on $\cH$ which $A$ outputs.
\end{theorem*}

\subsection{Overview of Proof Techniques}\label{sec:overview}

The \otb framework allows us to bound the generalization error of offline learning algorithms by the sum of the regret of an artificially constructed online learning algorithm and the normalized sum of the expected costs incurred by the online learner. In the i.i.d. setting, the costs incurred by the online learner form a martingale difference sequence~\citep{cesa2004generalization,Lugosi2023OnlinetoPACCG}. This simple but powerful observation allows one to upper bound the sum of expected costs using standard concentration inequalities. However, in the dependent data setting, the expected costs incurred by the online learner no longer form a martingale difference sequence. Consequently, it is no longer straightforward to apply concentration inequalities to bound the generalization error of the offline learning algorithms, as in the i.i.d. case of \citet{Lugosi2023OnlinetoPACCG}. 

To circumvent this issue, we rewrite the sum of expected costs of the online learning algorithm as the expected cost incurred by this online learner at time step $t+\tau$ with respect to its decision at step $t$, and a few remainder terms -- see \cref{lem:genbarbound3}. The first term is a so-called ``near-martingale" \citep{mohri10stab,Agarwal2011TheGA,duchi12ergodic}. More precisely, this term consists of a sum of random variables forming a martingale difference sequence and an additional expectation term that can be bounded using the mixing coefficients of the random process from which the training data is drawn. 

Finally, the remainder terms also consist of differences of expectations of costs incurred by the online learner at one time step with respect to either its output at a different time step or by the output of the offline learner. In order to bound these terms, we impose a doubly-Lipschitz condition on the loss function, boundedness of the observation and hypothesis spaces of the learning problem, and \emph{Wasserstein stability} (see \cref{assmp:stability}) of the online learner. The choice of a Wasserstein distance based stability criterion is motivated by its dual form (see Lemma~\ref{lem:kantrubdual}), which is used to bound the difference in expected cost incurred by the online learner in successive steps. This additional stability criterion comes for free as we show that the canonical EWA online learner used to instantiate our bound happens to be Wasserstein-stable -- see \cref{th:ftrL_stable} and \cref{cor:etaind}.

\subsection{Organization}
\cref{sec:related} compares and contrasts our work with a vast body of literature related to deriving generalization bounds for statistical learning algorithms. In particular, a detailed comparison with the concurrent work of \citet{neu2024delayed} is provided in \cref{sec:concurrent}. \cref{sec:prelims} introduces the mathematical prerequisites and sets up the problem. \cref{sec:main} develops the framework for obtaining generalization error upper bounds for statistical learning algorithms trained on data drawn from mixing processes. Following this, \cref{sec:applications} shows that the EWA learner conforms to the notion of stability introduced in this work and uses this fact to instantiate the generalization error bounds of \cref{sec:main}. The paper concludes with \cref{sec:conclusion}, which summarizes the contributions of this work and provides a few avenues for future work. Due to space constraints and better readability, proofs of certain technical results have been relegated to the Appendices.

\section{Related Works}\label{sec:related}
In learning with non-i.i.d. data, previous works have broadly focused on two approaches, via uniform convergence over complexity measures of the hypothesis space~\citep{yu1994rates,meir2000nonparametric,mohri2008}, or via data dependent bounds on mixing processes~\citep{mohri10stab,zhang19,fu2023sharper}. As argued before, the first class of bounds can be vacuous in overparameterized regimes, whereas the second approach suffered from the restrictive stability assumption on the offline learner. 

As argued previously, we get rid of the stability assumption on offline learners by employing the \otb framework. The \otb framework was introduced in the seminal work of \citet{cesa2004generalization}, which was later extended to give sharp generalization bounds for constrained linear classes based on complexity measures of the hypothesis classes by \citet{Kakade2008OnTG}. Recently, \citet{Lugosi2023OnlinetoPACCG} gave a framework that consolidates a vast quantity of the \otb literature under its umbrella to obtain generalization bounds for offline learning algorithms trained on i.i.d. data. 

Our work also introduces a new notion of algorithmic stability for online learners, which we call Wasserstein stability. Our definition of stability is incomparable to the more popular notions of algorithmic stability in the literature such as uniform stability \citep{bousquet2002, mohri10stab, fu2023sharper} for offline learners, or the stability notion for online learners introduced by \citet{Agarwal2011TheGA}. Our definition of stability is similar to the notion of \emph{one-step differential stability} introduced by \citet{ambuj2019}. Unlike the notions of stability referenced above, which are incompatible with differential privacy (DP), our notion of Wasserstein-stability holds promise for bridging the two seemingly disparate fields.
\subsection{Concurrent Work}\label{sec:concurrent}
Concurrent with our work, \citet{neu2024delayed} independently addressed the same problem in a draft updated on the arXiv in June 2024.\footnote{Our work was first uploaded to the arXiv in May 2024 -- see \citet{arxiv}.} Both works derive generalization bounds of the same order, i.e., $\frac{\regret}{n}+O(\frac{1}{\sqrt{n}})$, but involve different assumptions, and hence vastly different techniques, and also differ in the formulation of the \otb framework. 

Our work uses the standard \otb framework, whereas \citet{neu2024delayed} uses a variant of the \otb framework, using a delayed variant of the online learning problem where the learner starts seeing the cost of its choices only after a finite number of plays. This delayed \otb framework allows \citet{neu2024delayed} to avoid the stability requirement on the online learner, which is imposed by us. However, as noted earlier, this stability assumption essentially doesn't limit our choice of potential online learners, as the canonical EWA algorithm turns out to be stable. 

To circumvent the issue of having non-i.i.d. data, both works require assumptions on the mixing properties of the random process. We use a standard variant of the $\beta$ and the $\phi$ mixing assumptions (see, for example, Section~2 of \citet{Agarwal2011TheGA}) on the random process from which the dataset is drawn. On the other hand, \citet{neu2024delayed} uses a much stronger mixing assumption, applied directly on the associated loss function -- see Assumption~1 of \citet{neu2024delayed}. Our mixing assumption being weaker, imposes additional technical steps on our proofs -- see Lemmas~\ref{lem:lasttermexp} and \ref{lem:lastterm}. However, we require an additional Lipschitz assumption on the loss function which is not required by \citet{neu2024delayed} thanks to their already strong mixing assumption on the loss function. Further, both works require the loss function to be bounded. While we state the boundedness assumption explicitly, \citet{neu2024delayed} uses it implicitly through the use of Hoeffding's Lemma~\citep[Lemma~2]{neu2024delayed}. 
Finally, we note that our bounds seem to be more favorable when the KL divergence is high since in our bounds, we incur an additive term arising due to the specific definition of our mixing process, while in \citet{neu2024delayed}, they incur a multiplicative factor on the KL divergence term which arises due to their use of the delayed-\otb technique.

\section{Notation and Preliminaries}\label{sec:prelims}

We will sometimes interchangeably denote the expectation of a random variable $X$ w.r.t. a distribution $P$, i.e., $\expectover{X}{P}$ as $\inner{P}{X}$. 

\subsection{Information-theoretic Inequalities}
Let $P$ and $Q$ be two distributions on the same probability space $(\Omega,\filter)$, having densities $p$ and $q$ respectively with respect to an underlying measure $\mu$. Then the total variation distance $\dtv{\cdot}{\cdot}$ is defined w.r.t. $P$ and $Q$ as
{
\begin{align*}
    \dtv{P}{Q}&\defeq \underset{A\in\filter}{\sup}\abs{P(A)-Q(A)}= \frac{1}{2}\int_{\omega\in\Omega}\abs{p(\omega)-q(\omega)}\,d\mu(\omega).
\end{align*}
}
The Kullback–Leibler (KL) divergence between $P$ and $Q$, denoted by $\dkl{P}{Q}$ is defined as 
\begin{equation*}
\dkl{P}{Q}=\expectover{\ln\left(\frac{dP}{dQ}(X)\right)}{X\sim P}.
\end{equation*}
The relation between total variation distance and KL divergence is captured by Pinsker's inequality.
\begin{lem}[Pinsker's Inequality]\label{lem:pinsker}
    Every $P,Q$ on $(\Omega,\filter)$ satisfies
    $\dtv{P}{Q} \leq \sqrt{\frac{1}{2}\dkl{P}{Q}}.$
\end{lem}
The KL divergence can also be expressed as a variational form.
\begin{lem}\label{lem:donsker_varadhan}
    Let $X$ be a real-valued integrable random variable. Then for every $\lambda \in \mathbb{R}$,
    $$\log\expectover{e^{\lambda(X-\expecOver{P}(X))}}{P} = \sup_{Q\ll P} \biggl[\lambda\inner{Q-P}{X} - \dkl{Q}{P} \biggr].$$
\end{lem}

See Theorem 4.19 and Corollary 4.14 of \citet{Boucheron2013} for a proof of \cref{lem:pinsker} and \cref{lem:donsker_varadhan}, respectively.

\subsection{Probability Theory}
Consider any probability space $(\Omega,\filter,\mathbb{P})$, and let $(\filter_t)_{t\in\mathbb{N}}$ be a filtration.\footnote{See Chapter~9 of \citet{klenke} for a detailed exposition on filtrations and random processes adapted to a filtration.} Then, a random process $(M_t)_{t\in\mathbb{N}}$ adapted to the filtration $\filter_t$ is said to be a \emph{martingle difference sequence} if $\ex[|M_t|]<\infty$ and $\ex[M_t|\filter_{t-1}]=0$ almost surely. We state below the Azuma-Hoeffding inequality, which bounds the probability of the sum of the first $T$ terms of a martingale difference sequence exceeding some constant. 
\begin{lem}[Azuma-Hoeffding inequality]\label{lem:martineq}
    Let ${(M_t)}_{t\in \mathbb{N}}$ be a martingale difference sequence with respect to the filtration ${(\filter_t)}_{t\in \mathbb{N}}$. Let there be constants $c_t\in(0,\infty)$ such that $\forall t\geq 1$, almost surely $\abs{M_t}\leq c_t$. Then, $\forall \gamma>0$, $\mathrm{Pr}\left(\sum_{t=1}^TM_t\geq \gamma\right)\leq \mathrm{exp}\left(-\frac{\gamma^2}{2\sum_{t=1}^Tc_t^2}\right).$
\end{lem}
See Theorem~13.4 of \citet{MU}\footnote{Theorem~3.4 of \citet{MU} states the two-sided version of the inequality, i.e., for $\mathrm{Pr}(|\sum_{t=1}^TM_t|\geq \beta)$, and hence there is an additional factor of 2 in the RHS. 
Furthermore, Theorem~3.4 of \citet{MU} states the inequality for martingales and not martingale difference sequences, but the modification of the proof is straightforward.
} for a proof of \cref{lem:martineq}.

\subsection{Mixing Processes}
Consider a random process $(Z_t)_{t\in\mathbb{N}}$ with the probability distribution $\cP$. Let $\filter_t=\sigma(Z_1,\ldots,Z_t)$ denotes the smallest sigma-algebra generated by the set $\{Z_s\}_{s\in [t]}$. We denote by $\mathcal{P}^t_{[s]}=\mathcal{P}^t (\cdot ~|~\filter_s)$ the conditional probability distribution of $Z_t$ given the sigma-algebra $\filter_s$. 
In this work, we consider the case where the distribution of $Z_t$ converges (w.r.t. two different notions of convergence) to a stationary distribution $\mathcal{D}$ as $t\to \infty$, as defined below.

\begin{defn}[$\beta$ and $\phi$ coefficients]\label{def:mixingcoeff}
   The $\beta$ and $\phi$ mixing coefficients for the distribution $\cP$ are defined as
   \begin{equation}
       \beta(k) \defeq \underset{t\in\mathbb{N}}{\sup}\left\{ 2 ~\expectover{\dtv{\mathcal{P}^{t+k}_{[t]}}{\mathcal{D}}}{\mathcal{P}_{[t]}} \right\},
   \end{equation}
   \begin{equation}
       \phi(k) \defeq \underset{t\in\mathbb{N},B\in\filter_t}{\sup}\left\{ 2 ~\dtv{\mathcal{P}^{t+k}(\cdot|B)}{\mathcal{D}} \right\},
   \end{equation}
where the supremum in the definition of $\phi(k)$ is over elements of $\filter_t$ having non-zero measure, and $\mathcal{P}_{[t]}$ is the joint distribution of $Z_1,\ldots,Z_t$.
\end{defn}

\begin{defn}[$\beta$ and $\phi$ mixing]
    A stochastic process $\{Z_t\}_{t\in\mathbb{N}}$ is $\beta$-mixing (or $\phi$-mixing) if its distribution $\cP$ satisfies $\lim_{k\to\infty}\beta(k)= 0$ (respectively, if $\lim_{k\to\infty}\phi(k)= 0$).\label{def:mixing}
\end{defn}
\begin{remark}
It is trivial to see that for i.i.d. random processes, with $\cD$ being the per-letter marginal, the mixing coefficients satisfy $\beta(k)=\phi(k)=0$ for all $k\geq 1$. Hence, i.i.d. processes are both $\beta$ and $\phi$ mixing.    
\end{remark}
\begin{defn}[Geometric $\phi$-mixing]\label{def:geommix}
    Let $K,r>0$. A stochastic process is geometrically $\phi$-mixing with rate $K$ if $\phi(k)\leq K\cdot\exp{-k^r}$, for all $k>0$.
\end{defn}

We note at this point that there are no practical approaches to finding the decay rate of an unknown mixing process or even determining whether a stochastic process is mixing~\citep{yu1994rates,meir2000nonparametric} unless other properties (such as Gaussianity or Markovity) of the mixing process are known beforehand. There are, however, known examples of stochastic processes that are exponentially mixing (see~\citet{Mokkadem1988,meir2000nonparametric} for examples). 
\subsection{Wasserstein distances}
Let $(\cX,d)$ be any Polish space\footnote{A complete metric space is \textit{Polish} if it has a countable dense subset.}, and let $\tp$ and $\tq$ be any pair of probability measures on $\cX$. We denote by $\Pi(\tp,\tq)$ the set of joint measures on $\cX$ whose marginals are respectively $\tp$ and $\tq$. The \emph{Wasserstein distance} of order one between $\tp$ and $\tq$ is defined as 
\begin{equation}
    W(\tp,\tq)\triangleq \inf_{\pi\in\Pi(\tp,\tq)}\int_{\cX}d(x,y)d\pi(x,y). \label{eq:wasserstein}
\end{equation}

We now state the \emph{Kantorovich-Rubinstein duality formula} for Wasserstein distances of order one as given in Remark~6.5 of \citet{villani}.
\begin{lem}
    Let $\tp$ and $\tq$ be any pair of probability measures on a Polish space $(\cX,d)$. Then, 
    \begin{equation*}
        W(\tp,\tq)\;=\underset{\substack{\phi:\cX\to\reals\\\phi\text{ is 1-Lipschitz}}}{\sup}\biggl\{\int_\cX\phi\; d\tp-\int_\cX\phi\; d\tq\biggr\}.
    \end{equation*}
    \label{lem:kantrubdual}
\end{lem}

As an immediate consequence of \cref{lem:kantrubdual}, we have the following corollary.

\begin{restatable}{corollary}{glip}
    Let $\tp$ and $\tq$ be any pair of probability measures on a Polish space $(\cX,d)$. Let $\phi:\cX\to\reals$ be $G$-Lipschitz. Then,
    $
    W(\tp,\tq)\geq \frac{1}{G}\biggl[\int_\cX\phi\; d\tp-\int_\cX\phi\; d\tq\biggr].
    $
    \label{corr:Glip}
\end{restatable}
\begin{proof}
    Define a new function $\psi:\cX\to\reals$ as $\psi(x)=\frac{1}{G}\phi(x)$, and note that $\psi$ is $1$-Lipschitz by definition. The result now follows using \cref{lem:kantrubdual}.
\end{proof}

Next, we present an inequality that relates Wasserstein distances of order $1$ to the Total Variation distance.\footnote{ See Particular Case 6.16 of \citet{villani}.}
\begin{lem}\label{lem:wasserstein_dtv_bound}
    Let $P$ and $Q$ be two probability measures on a Polish space $(\cX,d)$. If the diameter of the underlying metric space is bounded by $M\geq 0$, then $W(P, Q) \leq M \cdot \dtv{P}{Q}.$
\end{lem}

\subsection{Problem Setup}
Consider a measurable instance space $\cZ$. Let the set $S_n=(Z_1,\ldots,Z_n)$ denoted as a \textit{training set}, be a tuple of $n$ random variables (not necessarily independent), drawn from some random process $Z_1,Z_2,\cdots$ over $\cZ$, with a probability distribution $\cP$ which mixes to a stationary distribution $\cD$, as defined in \cref{def:mixing}.
\begin{assumption}\label{assmp:instance}
    We assume that $\cZ$ is equipped with a metric $\nZ{\cdot}$, and its diameter is $\rz$. 
\end{assumption}

A \textit{learning algorithm} $A:\cZ^n\mapsto \mathcal{H}$ maps (in a randomized fashion) any such $n$ tuple to an element $H^*=A(S_n)$ in a measurable set $\cH$, where $\cH$ is known as the \textit{hypothesis class}. The performance of the learning algorithm $A$ is measured with respect to a loss function $\ell:\cH\times\cZ\mapsto\mathbb{R}_{+}$. We now state our assumptions with respect to the hypothesis space $\cH$ and the loss function $\ell$ below.

\begin{assumption}[Assumptions on the Hypothesis Space $\cH$]\label{assmp:hypothesis}
We assume that the space $\cH$ is  equipped with the metric $\nH{\cdot}$, and $\cH$ is \emph{Polish} with respect to $\nH{\cdot}$. Moreover, the diameter of $\cH$ is $\rh$. We denote by $\Delta_\cH$ the set of all distributions over $\cH$.
\end{assumption}

\begin{assumption}[Doubly Lipschitz Loss]\label{assmp:lipschitz}
    Let  $\ell:\cH\times\cZ\mapsto\mathbb{R}_{+}$, be a loss function. We assume that $\ell$ is \emph{doubly Lipschitz}. More precisely, $\ell$ is $\gh$-\emph{Lipschitz} w.r.t the first argument $h$, and it is $\gz$-Lipschitz w.r.t. the second argument $z$. Since $\ell$ is doubly-Lipschitz and both of its inputs are from bounded domains, the range of $\ell$ is also bounded, as shown in the following lemma.
\end{assumption}

\begin{restatable}{lem}{lossbound}\label{lem:lossbound}
    For any $(h,z)\in\cH\times\cZ$, the loss function $\ell$ satisfies $|\ell(h,z)|\leq B_\ell$, where $B_\ell\defeq\inf_{(h',z')\in\cH\times\cZ}|\ell(h',z')|+\gh\rh+\gz\rz$.
\end{restatable}
\begin{proof}
    {Fix any} $(h',z')\in\cH\times\cZ$. On applying triangle inequality, we have $\biggl||\ell(h,z)|-|\ell(h',z')|\biggr|\leq |\ell(h,z)-\ell(h',z')|$. Now, noting that $\ell$ is doubly-Lipschitz, and $\cH$ and $\cZ$ have diameters $\rh$ and $\rz$ respectively, we have
    \begin{align*}
        |\ell(h,z)-\ell(h',z')| & = |\ell(h,z)-\ell(h,z')+\ell(h,z')-\ell(h',z')|\\
                                & \leq |\ell(h,z)-\ell(h,z')|+|\ell(h,z')-\ell(h',z')|\\
                                & \leq \gz\rz+\gh\rh.
    \end{align*}
    Therefore, we have $|\ell(h,z)|\leq |\ell(h',z')|+\gz\rz+\gh\rh$. The lemma is then proved by taking an infimum over $(h',z')\in\cH\times\cZ$.
\end{proof}

The \textit{training error} of the learning algorithm $A$ is the cumulative loss conceded by $A$ over its training set $S_n$: $\sum_{i=1}^{n}\ell(H^*,Z_i)$, where $H^*=A(S_n)$. The \textit{test error} of $A$ is defined as the expected loss of the learning algorithm over any instance from the instance space: $\expectover{\ell(H^*,Z^{\prime})}{Z^{\prime}\sim \mathcal{D}}$. 

\begin{defn}[Generalization Error of Offline Learners]
    The \textit{overfitting error} of $A$ is defined as the difference between the test error and the mean training error of $A$, as 
    $$\gen(A,S_n)= \expectover{\ell(H^*,Z^{\prime})}{Z^{\prime}\sim \mathcal{D}}-\frac{1}{n}\sum_{t=1}^{n}\ell(H^*,Z_t).$$
    The generalization error of a fixed statistical learning algorithm $A$ is defined as
    $$\genbar(A, S_n) \defeq \ex_{H\sim\tp_{H^*}}[\gen(H,S_n) ~|~ S_n],$$
where $\tp_{H^*}\defeq\tp_{A(S_n)}$, i.e., the conditional distribution of the output $H^*$ produced by $A$ given training set $S_n$. 
\end{defn}

\subsection{The \textsc{\otb} framework}\label{sec:Learninggame}
An interesting paradigm for evaluating the generalization ability of statistical learning algorithms studied in the literature is \otb conversions~\citep{cesa2004generalization}. In the \otb setting,  a connection is established between the performance of batch learners (on unknown instances) and the performance of online learning algorithms (on known instances). We first introduce the online learning setting and subsequently describe the \otb paradigm. 
\subsubsection{Overview of Online Learning}
The online learning setting can be modeled as the following two-player game, henceforth referred to as the \otb game, between a learning algorithm $\learner$ and an adversary. The learner $\learner$ has sequential access to a stream of data generated from an arbitrary source, and at every time step $t$, based on decisions taken up to the $t-1$\textsuperscript{th} time-step, the learner tries to predict the correct label of the next data point and incurs a loss that is decided by an adversary. The goal of the online learning setup is to minimize some notion of \textit{regret}, i.e., the loss incurred at the $t$\textsuperscript{th} time step by the online learner should be reasonably close to the loss incurred by the best possible offline learner that has access to the data points up to the $(t-1)$\textsuperscript{th} time-step.

An example of an online learning game is the Hedging algorithm or the Exponential-Weighted Average (EWA) algorithm~\citep{LITTLESTONE1994,FREUND1997119,vovk98}. In EWA, we first fix a data-free prior distribution $\cP_1\in\Delta_{\cH}$ and a learning rate parameter $\eta>0$. At every iteration $t>0$, we perform the following updates: 
\begin{equation}\label{eq:ewa_update}
    \cP_{t+1} \defeq \underset{\cP \in \deltaH}{argmin} \left\{ \inner{\cP}{c_t} + \frac{1}{\eta} \dkl{\cP}{\cP_t} \right\}.
\end{equation}
\subsubsection{The \textsc{\otb} Game}
In the \otb conversion game, we assume that the instances $Z_t,t\in[n]$ in the offline setting are provided to the online learner $\learner_n$. We now describe the generalization game from \citet{Lugosi2023OnlinetoPACCG}, played over $n$ rounds below.

\begin{game}[The \otb Generalization game]
\phantom{}\\
$\triangleright~$At the $t$\textsuperscript{th} iteration,
\begin{enumerate}[noitemsep,nolistsep]
    \item The online learning algorithm $\learner_n$ chooses a distribution $\tp_t\in\Delta_{\mathcal{H}}$ over the hypothesis space, with knowledge of only $Z_1,\ldots,Z_{t-1}$.
    \item The adversary picks a cost function $c_t\from \mathcal{H}\to \mathbb{R}$ for each hypothesis $h\in\mathcal{H}$ as 
    $$c_t(h)\defeq \ell(h,Z_t)- \expectover{\ell(h,Z^{\prime})}{Z^{\prime}\sim \mathcal{D}}.$$
    \item The online learning algorithm $\learner_n$ incurs a cost $\langle \tp_t, c_t\rangle \defeq \ex_{\tp_t}[c_t(H_t)].$
    \item The adversary reveals to the online learner the sample $Z_t$. Now the online learner can compute the cost function.
    \end{enumerate}
    \label{game}
    \end{game}
Recall that $\tp_{H^*} \defeq \tp_{A(S_n)}$. Then the regret of the learning algorithm $\learner_n$ with respect to the comparator distribution $\tp_{H^*}$ over the hypothesis space is:
\begin{equation}\label{eq:regret}
    \reg{\learner_n}{A}(P_{H^*})\defeq \sum_{t=1}^{n} \langle \tp_t - \tp_{H^*},c_t \rangle,
\end{equation}
Henceforth, we shall drop the comparator distribution $P_{H^*}$ from the notation of regret for the sake of brevity. We now present the following technical lemma that bounds the cost function from \cref{game}.

\begin{restatable}{lem}{lipschitz}\label{lem:lipschitz}
    For any fixed instance of $Z_t$, and any $h_1,h_2\in\cH$, the cost function $c_t(\cdot)$ picked by the adversary in the generalization game satisfies $|c_t(h_1)-c_t(h_2)|\leq 2\gh\rh$. On the other hand, for any fixed $h\in\cH$, and any $t,t'$ and a fixed realization of $Z_t, Z_{t'}$, we have $|c_t(h)-c_{t'}(h)|\leq\gz\rz$.
\end{restatable}
\begin{proof}
    {Let} $h_1,h_2\in\mathcal{H}$. Then, for any fixed instance of $Z_t$, we have
     \begin{align}
         |c_t(h_1)-c_t(h_2)| & = |\ell(h_1,Z_t)- \expectover{\ell(h_1,Z^{\prime})}{Z^{\prime}\sim \mathcal{D}}-\ell(h_2,Z_t)+\expectover{\ell(h_2,Z^{\prime})}{Z^{\prime}\sim \mathcal{D}}|\\
         & 
         {\leq} |\ell(h_1,Z_t)-\ell(h_2,Z_t)|+|\expectover{\ell(h_1,Z^{\prime})}{Z^{\prime}\sim \mathcal{D}}-\expectover{\ell(h_2,Z^{\prime})}{Z^{\prime}\sim \mathcal{D}}|\label{eq:lipschitza}\\
         & 
         {\leq} \gh\nH{h_1-h_2}+\expectover{|\ell(h_1,Z')-\ell(h_2,Z')|}{Z'\sim \mathcal{D}}
         \label{eq:lipschitzb}\\
         & 
         {\leq} 2\gh\nH{h_1-h_2}
         \label{eq:lipschitzc}\\
         & 
         {\leq} 2\gh\rh\label{eq:lipschitzd},
     \end{align}
     where \cref{eq:lipschitza,eq:lipschitzb} use the triangle inequality, and  \cref{eq:lipschitzb,eq:lipschitzc}use the fact that $\ell(\cdot,\cdot)$ is $\gh$-Lipschitz in the first argument, and \cref{eq:lipschitzd} uses the bounded diameter of the hypothesis space $\cH$.

     On the other hand, for any $h\in\cH$ and any $t,t'$ and any instance of $Z_t,Z_{t'}$,
     \begin{align}
          |c_t(h)-c_{t'}(h)| & = |\ell(h,Z_t)- \expectover{\ell(h,Z^{\prime})}{Z^{\prime}\sim \mathcal{D}}-\ell(h,Z_{t'})+\expectover{\ell(h,Z^{\prime})}{Z^{\prime}\sim \mathcal{D}}|\\
          & 
          {\leq} \gz\nZ{Z_t-Z_{t'}}\label{eq:lipschitze}\\
          &
          {\leq} \gz\rz\label{eq:lipschitzf},
     \end{align}
     where \cref{eq:lipschitze} follows from the fact that $\ell(\cdot,\cdot)$ is $\gz$-Lipschitz in the second argument, and \cref{eq:lipschitzf} follows by noting that the diameter of $\cZ$ is $\rz$.
 \end{proof}

In order to bound the generalization error of $A$ using the regret of the online learner $\learner_n$ for \cref{game}, we shall require the $\learner_n$ to be \emph{Wasserstein-stable} as defined below.
\begin{defn}[Wasserstein Stable]\label{assmp:stability}
    Given a non-increasing sequence $\kappa(t), t\geq 1$, an online learning algorithm is said to be $\kappa(t)$-Wasserstein-stable if for any $t\in[n]$, the online learner $\learner_n$ satisfies
\begin{equation}
    W(\tp_t,\tp_{t+1})\leq \kappa(t). \label{eq:stability}
\end{equation}
We refer to $\kappa(t)$ as the stability parameter at round $t$.
\end{defn}
The following technical lemma allows us to bound the sum of differences in expected costs between Wasserstein-stable online learning algorithms for \cref{game} and the offline comparator.
\begin{restatable}{lem}{taushift}\label{lem:taushift}
    For any Wasserstein-stable online learning algorithm $\learner_n$ for \cref{game} and any $\tau=o(n)$, the following bound holds with probability one.
     \begin{align*}
         &\sum_{t=1}^n\left[\expectover{c_{t+\tau}(H)}{H\sim\tp_t}-\expectover{c_{t+\tau}(H)}{H\sim\tp_{H^*}}\right]\leq\regret_{\learner_n,A}+2\gh\tau\sum_{t=1}^n\kappa(t)+4\tau \gh\rh.
     \end{align*}
\end{restatable}
\begin{proof}
     Observe that the LHS term $\sum_{t=1}^n\biggl[ \expectover{c_{t+\tau}(H)}{H\sim\tp_t}-\expectover{c_{t+\tau}(H)}{H\sim\tp_{H^*}}\biggr]$ in \cref{lem:taushift} can be rewritten as
     \begin{align}
        & \underbrace{\sum_{t=1}^n\biggl[\expectover{c_t(H)}{H\sim\tp_t}-\expectover{c_t(H)}{H\sim\tp_{H^*}}\biggr]}_{T_1}+\underbrace{\sum_{t=1}^{n-\tau}\biggl[\expectover{c_{t+\tau}(H)}{H\sim\tp_t}-\expectover{c_{t+\tau}(H)}{H\sim\tp_{t+\tau}}\biggr]}_{T_2}\notag\\
        &\hspace{0.4cm}+\underbrace{\sum_{t=n-\tau+1}^n\expectover{c_{t+\tau}(H)}{H\sim\tp_t}-\sum_{t=n+1}^{n+\tau}\expectover{c_t(H)}{H\sim\tp_{H^*}}}_{T_3}+\underbrace{\sum_{t=1}^\tau \expectover{c_t(H)}{H\sim\tp_{H^*}}-\sum_{t=1}^\tau\expectover{c_t(H)}{H\sim\tp_t}}_{T_4}. \label{eq:taushift:1}
     \end{align}
    We now proceed to bound each term individually as follows.
    Note that by definition, we have $T_1\leq \regret_{\learner_n,A}$. 
    
    Next, we proceed to bound the $T_2$ term as follows.
    
    \begin{align}
        T_2\; 
        &{\leq }\;\sum_{t=1}^{n-\tau}2\gh W(\tp_t,\tp_{H_{t+\tau}})\label{eq:taushifta}\\ 
        &{\leq }\sum_{t=1}^{n-\tau}2\gh\sum_{r=0}^{\tau-1}W(\tp_{t+r},\tp_{t+r+1})\label{eq:taushiftb}\\
            &{\leq } \sum_{t=1}^{n-\tau}2\gh \sum_{r=0}^{\tau-1}\kappa(t+r)\label{eq:taushiftc}\\
            &{\leq } 2\gh\tau\sum_{t=1}^{n-\tau}\kappa(t)\label{eq:taushiftd}\\ 
            &{\leq } 2\gh\tau\sum_{t=1}^n\kappa(t)\label{eq:taushifte},
    \end{align}
    
    where \cref{eq:taushifta} follows from Corollary~\ref{corr:Glip} and the fact that $c_t$ is $2\gh$-Lipschitz (see Lemma~\ref{lem:lipschitz}), \cref{eq:taushiftb} follows using the triangle inequality, \cref{eq:taushiftc} uses the stability assumption of the learner $\learner_n$, \cref{eq:taushiftd} uses the fact that $\kappa(\tau)$ is non-increasing, and \cref{eq:taushifte} uses the non-negativity of $\kappa(\tau)$ which follows from \cref{eq:stability}. 
    Next, we bound $T_3$ as follows.
    \begin{align*}
        \sum_{t=n-\tau+1}^n  \expectover{c_{t+\tau}(H)}{H\sim\tp_t}-\sum_{t=n+1}^{n+\tau}\expectover{c_t(H)}{H\sim\tp_{H^*}} & = \sum_{t=n+1}^{n+\tau}\biggl[\expectover{c_{t}(H)}{H\sim\tp_{H_{t-\tau}}}-\expectover{c_t(H)}{H\sim\tp_{H^*}}\biggr]\\
        & = \sum_{t=n+1}^{n+\tau}\expectover{c_t(H_1)-c_t(H_2)}{\substack{{H_1\sim\tp_{H_{t-\tau}}}\\{H_2\sim\tp_{H^*}}}}\\
        & \leq \sum_{t=n+1}^{n+\tau}2\gh\rh\\
        &= 2\tau \gh\rh,
    \end{align*}
    where the penultimate step uses the fact that $c_t$ is $2\gh$-Lipshitz via Lemma~\ref{lem:lipschitz}, and that the diameter of $\cH$ is $\rh$. Similarly, one can bound $T_4\leq 2\tau \gh\rh$. Plugging in all the bounds in \cref{eq:taushift:1}, we have the result.
 \end{proof}

\section{Generalization Error Bounds for Mixing processes}\label{sec:main}
In this section, we state and prove our main results on the generalization error of statistical learning algorithms trained on training samples drawn from a mixing process. Our first result is an upper bound on the expected generalization error in terms of the expected regret of an online learner $\learner_n$ for \cref{game}.
\begin{theorem}[Expected generalization error]
    For any arbitrary Wasserstein-stable online learner $\learner_n$ for \cref{game}, and any $\tau=o(n)$, the expected generalization error $\expect{\genbar(A,S_n)}$ of the learning algorithm $A$ with input $S_n=(Z_1,\ldots,Z_n)$ drawn from the mixing random process $\{Z_t\}_{t\in\mathbb{N}}$ is upper bounded by
    \begin{align*}
        \frac{1}{n}\expect{\regret_{\learner_n,A}}+\frac{2\tau\gh}{n}\left(\sum_{t=1}^n\kappa(t)+2\rh\right)+\frac{\tau\gz\rz}{n}+B_{\ell}\cdot\beta(\tau+1).
    \end{align*}
    \label{th:exgenbarbound1}
\end{theorem}
The next theorem complements Theorem~\ref{th:exgenbarbound1} by providing a high probability upper bound on the generalization error of the learning algorithm $A$ in terms of the regret of any learner $\learner_n$ for \cref{game}.

\begin{theorem}[Generalization Error]
    For any arbitrary Wasserstein-stable online learner $\learner_n$ for \cref{game}, and any $\tau=o(n)$, $\delta>0$, the generalization error $\genbar(A,S_n)$ of the learning algorithm $A$ with input $S_n=(Z_1,\ldots,Z_n)$ drawn from the mixing random process $\{Z_t\}_{t\in\mathbb{N}}$ is upper bounded, with probability at least $1-\delta$, by
    \begin{align*}
        \frac{\regret_{\learner_n,A}}{n}+\frac{2\tau\gh}{n}\left(\sum_{t=1}^n\kappa(t)+2\rh\right)+\frac{\tau\gz\rz}{n}+2\gh\rh\sqrt{\frac{{2\tau\log{(\tau/\delta)}}
    }{{n}}}+B_{\ell}\cdot\phi(\tau+1).
    \end{align*}
    \label{th:genbarbound1}
\end{theorem}

In order to establish \cref{th:exgenbarbound1} and \cref{th:genbarbound1}, we first prove the following intermediate lemma which relates the generalization error of the offline learner $A$ to the regret of the online learner $\learner_n$.

 \begin{restatable}{lem}{genbarbound}\label{lem:genbarbound3}
    For any $\tau=o(n)$, and any Wasserstein-stable online learner $\learner_n$ for \cref{game}, with probability one, $\genbar(A,S_n)$ is at most
    \begin{align*}
        \frac{\regret_{\learner_n,A}}{n}+\frac{\tau}{n}\biggl[2\gh\sum_{t=1}^n\kappa(t)+4\gh\rh+\gz\rz\biggr]-\frac{1}{n}\sum_{t=1}^{n} \expectover{c_{t+\tau}(H)}{H\sim\tp_t}.
    \end{align*}
\end{restatable}
\begin{proof}
    \begin{align}
         \genbar(A,S_n) &\defeq \expectover{\gen(H,S_n) ~|~ S_n}{H\sim\tp_{H^*}}\\
         &=\expectover{\expectover{\ell(H,Z^{\prime})}{Z^{\prime}\sim \mathcal{D}}-\frac{1}{n}\sum_{t=1}^n\ell(H,Z_t) ~|~ S_n}{H\sim\tp_{H^*}}\\
         &=-\frac{1}{n}\sum_{i=1}^n\expectover{c_t(H)}{H\sim\tp_{H^*}}\\
         &=\frac{1}{n}\sum_{t=1}^{n}\biggl[\expectover{c_{t+\tau}(H)}{H\sim\tp_t}-\expectover{c_{t+\tau}(H)}{H\sim\tp_{H^*}}\biggr]-\frac{1}{n}\sum_{t=1}^{n} \expectover{c_{t+\tau}(H)}{H\sim\tp_t}\\
         &\hspace{0.5cm}+\frac{1}{n}\sum_{t=1}^n\biggl[\expectover{c_{t+\tau}(H)}{H\sim\tp_{H^*}}-\expectover{c_t(H)}{H\sim\tp_{H^*}}\biggr]\\
         & 
         {\leq}\frac{1}{n}\regret_{\learner_n,A}+\frac{\tau}{n}\biggl[2\gh\sum_{t=1}^n\kappa(t)+4\gh\rh\biggr]-\frac{1}{n}\sum_{t=1}^{n} \expectover{c_{t+\tau}(H)}{H\sim\tp_t}\\
         &\hspace{0.5cm}+\frac{1}{n}\sum_{t=1}^\tau\biggl[\expectover{c_{n+t}(H)}{H\sim\tp_{H^*}}-\expectover{c_t(H)}{H\sim\tp_{H^*}}\biggr]\label{eq:genbarbound3a}\\
         &
         {\leq} \frac{1}{n}\regret_{\learner_n,A}+\frac{\tau}{n}\biggl[2\gh\sum_{t=1}^n\kappa(t)+4\gh\rh+\gz\rz\biggr]-\frac{1}{n}\sum_{t=1}^{n} \expectover{c_{t+\tau}(H)}{H\sim\tp_t}\label{eq:genbarbound3b},
     \end{align}
     where \cref{eq:genbarbound3a} uses \cref{lem:taushift} and \cref{eq:genbarbound3b} uses \cref{lem:lipschitz}.
\end{proof}
To complete the proofs of \cref{th:exgenbarbound1} and \cref{th:genbarbound1}, we upper bound the final term in \cref{lem:genbarbound3}, respectively, in expectation and with high probability. This is accomplished in the following pair of lemmas which rearranges the terms in $\sum_{t=1}^{n} \expectover{c_{t+\tau}(H)}{H\sim\tp_t}$ as sums of random variables forming a martingale difference sequence, and a remainder term which can be bounded using the mixing coefficients for the random process $\{Z_t\}_{t\in\mathbb{N}}$. 

 \begin{restatable}{lem}{lasttermexp}\label{lem:lasttermexp}
    For any $\tau=o(n)$,
    \begin{align*}
        -\sum_{t=1}^{n} \expectover{\expectover{c_{t+\tau}(H)}{H\sim\tp_t}}{\mathcal{P}}\leq n\cdot B_\ell\cdot\beta(\tau+1).
    \end{align*}
\end{restatable}
\begin{restatable}{lem}{lastterm}
    \label{lem:lastterm}
    With probability at least $1-\delta$, for any $\tau=o(n), \delta>0$,
    \begin{align*}
        -\sum_{t=1}^{n} \expectover{c_{t+\tau}(H)}{H\sim\tp_t}\leq 2\gh\rh\sqrt{2n\tau\log{(\tau\delta)}}+n\cdot B_\ell\cdot\phi(\tau+1).
    \end{align*}
\end{restatable}
\begin{proof}
    First, we rearrange the terms in $ \sum_{t=1}^{n} \expectover{c_{t+\tau}(H)}{H\sim\tp_t}$ as follows. Consider the indices $a\in\{1,\ldots,\tau\}$, and $b\in\{1,\ldots,i_a\}$, where $i_a\defeq\min\{b':(b'-1)\tau+a\leq n\}$, and note that $i_a\leq\lceil n/\tau \rceil$ for any $1\leq a\leq \tau$. Now, let $\filter_{(b-1)\tau+a-1}\defeq\sigma(Z_1,\ldots,Z_{(b-1)\tau+a-1})$, and define $X^a_b\defeq -\expectover{c_{b\tau+a}(H)}{H\sim\tp_{(b-1)\tau+a}}$. 
    We rewrite the term $Y\defeq-\sum_{t=1}^{n} \expectover{c_{t+\tau}(H)}{H\sim\tp_t}$ as follows:
    
    \begin{equation}\label{eq:expansion2}
        Y=\sum_{a=1}^{\tau}\underset{\mathcal{M}_a}{\underbrace{\sum_{b=1}^{i_a}\left(X^a_b-\expectover{X^a_b ~|~ \filter_{(b-1)\tau+a-1}}{\cP^{b\tau+a}}\right)}}
        +
        \sum_{a=1}^{\tau}\sum_{b=1}^{i_a}\left(\expectover{X^a_b ~|~ \filter_{(b-1)\tau+a-1}}{\cP^{b\tau+a}}\right).
    \end{equation} 

    Firstly, note that $\expectover{\mathcal{M}_a}{\mathcal{P}}=0$ for all $1\leq a\leq \tau$. Therefore, 
    
    \begin{equation}
        \expectover{Y}{\mathcal{P}}=\sum_{a=1}^{\tau}\sum_{b=1}^{i_a}\expectover{\expectover{X^a_b ~|~ \filter_{(b-1)\tau+a-1}}{\cP^{b\tau+a}}}{\mathcal{P}}.\label{eq:expectY}
    \end{equation}

    Now, let $\text{p}_{[s]}^t$ and $\text{d}$ be the densities with respect to some measure $\mu$.\footnote{For example, $\mu$ can be chosen as $\tp_{[s]}^t+\mathcal{D}$.} Then, the second term in \cref{eq:expansion2} can be rewritten as follows:
    
    \begin{align}\label{eq:nearmartingalepart}
            \expectover{X^a_b ~|~ \filter_{(b-1)\tau+a-1}}{\cP^{b\tau+a}} &=\expectover{-\expectover{c_{b\tau+a}(H)}{\tp_{(b-1)\tau+a}} ~\biggl|~ \filter_{(b-1)\tau+a-1}}{\cP^{b\tau+a}}\\
            &\leq{\expectover{\expectover{\abs{c_{b\tau+a}(H)}}{\tp_{(b-1)\tau+a}} ~\biggl|~ \filter_{(b-1)\tau+a-1}}{\cP^{b\tau+a}}}\\
            &
            {\leq}\expectover{\int_{\cZ}\ell(H,Z)\cdot
            \abs{\text{p}^{b\tau+a}_{\left[(b-1)\tau+a-1\right]}-\text{d}}\,d\mu}{\tp^{(b-1)\tau+a}_{\left[{(b-1)\tau+a-1}\right]}}\label{eq:nearmartingaleparta}\\
            &
            {\leq} 2\cdot B_\ell\cdot\dtv{\tp^{b\tau+a}_{\left[(b-1)\tau+a-1\right]}}{\cD}\label{eq:nearmartingalepartb},
    \end{align}
  where \cref{eq:nearmartingaleparta} follows via Fubini's Theorem (Theorem~14.19 of \citet{klenke}) and by noting the fact that the distribution $\tp_{(b-1)\tau+a}$ returned by the online-learner is independent of $Z_{b\tau+a}$ conditioned on $Z_1,\ldots,Z_{(b-1)\tau+a-1}$, and \cref{eq:nearmartingalepartb} uses \cref{lem:lossbound}. Then, via Definition~\ref{def:mixingcoeff}, we have 
  $$
  \expectover{\expectover{X^a_b ~|~ \filter_{(b-1)\tau+a-1}}{\cP^{b\tau+a}}}{\mathcal{P}}\leq B_\ell\cdot \beta(\tau+1).
  $$
  Plugging the above in \eqref{eq:expectY} completes the proof.
\end{proof}
We can now complete the proofs of Theorems~\ref{th:exgenbarbound1} and \ref{th:genbarbound1}.
 \begin{proof}[Proofs of Theorems~\ref{th:exgenbarbound1} and \ref{th:genbarbound1}]
     Theorem~\ref{th:exgenbarbound1} now follows by plugging in \cref{lem:lasttermexp} in \cref{lem:genbarbound3}. Similarly, Theorem~\ref{th:genbarbound1} follows by plugging \cref{lem:lastterm} in \cref{lem:genbarbound3}.
 \end{proof}
\section{Applications: Generalization Error Bounds for EWA}\label{sec:applications}
In this section, we obtain generalization bounds for data drawn from mixing processes by using the EWA learner (see \cref{eq:ewa_update}) as our online learning strategy. First, we prove that the EWA Learner is Wasserstein-stable. 
\begin{theorem}\label{th:ftrL_stable}
Under Assumptions~\ref{assmp:instance},~\ref {assmp:hypothesis} and ~\ref{assmp:lipschitz}, the Exponentially Weighted Averages (EWA) algorithm is a Wasserstein-stable online learner (see \cref{assmp:stability}), where the stability parameter $\kappa(t)\leq \eta \gh \rh^2$ for all $t\in[n]$, and $\eta>0$ is the learning rate of the EWA learner.
\end{theorem}

To prove \cref{th:ftrL_stable}, we require the following lemmas.

\begin{restatable}[EWA Minimizer]{lem}{minequal}\label{lem:minimizer_equality}
    Let $\eta >0$ be a \textit{learning rate}. Then, \cref{eq:ewa_update} is minimized by a distribution $\cP^*\ll \cP_t$ s.t. $\cP^*$ satisfies the following equality:
    \begin{equation}\label{eq:ewamin}
        \inner{\cP^*}{c_t} + \frac{1}{\eta}\dkl{\cP^*}{\cP_t} 
        = -\frac{1}{\eta}\log\left(\expecOver{H \sim \cP_t}\left[ e^{-\eta c_t(H)}\right]\right).
    \end{equation}
\end{restatable}
\begin{proof}
    Seting $\cP$ to $\cP_t$, $Q$ to $\cP$, and $Z$ to $c_t(H)$ in \cref{lem:donsker_varadhan}, we obtain for all $\cP \ll  \cP_t$:
    $
    \log\biggl(\expecOver{H \sim \cP_t}\biggl[ e^{-\eta\left(c_t(H) - \inner{\cP_t}{c_t}\right)}\biggr]\biggr)$ $\geq -\eta \biggl( \inner{P}{c_t} - \inner{\cP_t}{c_t}\biggr) - \dkl{\cP}{\cP_t}.
    $
    Rearranging the terms and dividing by $\eta$, we get the following:
    \begin{align*}\label{}
        \inner{\cP}{c_t} + \frac{1}{\eta}\dkl{\cP}{\cP_t} & \geq \inner{\cP_t}{c_t} - \frac{1}{\eta}\log\left(\expecOver{H \sim \cP_t}\left[ e^{-\eta\left(c_t(H) - \inner{\cP_t}{c_t}\right)}\right]\right)\\
        & = \inner{\cP_t}{c_t} - \frac{1}{\eta}\log\left(\expecOver{H \sim \cP_t}\left[ e^{-\eta c_t(H)} \cdot e^{\eta\inner{\cP_t}{c_t}}\right]\right) \\
        & = \inner{\cP_t}{c_t} - \frac{1}{\eta}\log\left(e^{\eta\inner{\cP_t}{c_t}} \cdot \expecOver{H \sim \cP_t}\left[ e^{-\eta c_t(H)}\right]\right) \\
        & = \inner{\cP_t}{c_t} - \inner{\cP_t}{c_t} - \frac{1}{\eta}\log\left( \expecOver{H \sim \cP_t}\left[ e^{-\eta c_t(H)}\right]\right) \\ 
        & =  - \frac{1}{\eta}\log\left( \expecOver{H \sim \cP_t}\left[ e^{-\eta c_t(H)}\right]\right).
    \end{align*}    
    This implies that
 \begin{equation}\label{eq:etaformminimizer}
        \inner{\cP}{c_t} + \frac{1}{\eta}\dkl{\cP}{\cP_t}\geq -\frac{1}{\eta}\log\left(\expecOver{H \sim \cP_t}\left[ e^{-\eta c_t(H)}\right]\right).
    \end{equation}

    Since the above inequality holds for all $\cP \ll  \cP_t$, we now show that for a distribution $\cP^{\prime}$ s.t. $\cP^{\prime}\ll \cP_t$, equality is attained when the following condition holds. 
    \begin{equation}\label{eq:minimizer_condition}
        \frac{d\cP^{\prime}}{d\cP_t}(H) = \frac{e^{-\eta c_t(H)}}{\int_{\cH} e^{-\eta c_t(H')} d\cP_t(H') }.
    \end{equation}

    If we plug $\cP^{\prime}$ from \cref{eq:minimizer_condition} into \cref{eq:etaformminimizer}, the LHS of \cref{eq:etaformminimizer} becomes
    \begin{align*}
        \inner{\cP^{\prime}}{c_t} + \frac{1}{\eta}\dkl{\cP^{\prime}}{\cP_t}
        &=\inner{\cP^{\prime}}{c_t} + \frac{1}{\eta}\int_{\cH}\log\left(\frac{d\cP^{\prime}}{d\cP_t}(H)\right)d\cP^{\prime}(H) \\
        &=\inner{\cP^{\prime}}{c_t} + \frac{1}{\eta}\int_{\cH}\log\left( \frac{e^{-\eta c_t(H)}}{\int_{\cH} e^{-\eta c_t(H')} d\cP_t(H') } \right)d\cP^{\prime}(H) \\
        &=\inner{\cP^{\prime}}{c_t} - \int_{\cH}c_t(H)d\cP^{\prime}(H) - \frac{1}{\eta}\int_{\cH}\log\left( \int_{\cH} e^{-\eta c_t(H')} d\cP_t(H') \right)d\cP^{\prime}(H) \\
        &=-\frac{1}{\eta}\log\left( \int_{\cH} e^{-\eta c_t(H')} d\cP_t(H') \right) \\ 
        &=-\frac{1}{\eta}\log(\expecOver{H\sim \cP_t}\left[e^{-\eta c_t(H)}\right]).
    \end{align*}
    Hence, the minimizer for \cref{eq:etaformminimizer} exists and is attained when \cref{eq:minimizer_condition} holds. 
\end{proof}

We can now complete the proof of \cref{th:ftrL_stable}.
\begin{proof}[Proof of \cref{th:ftrL_stable}]
    Applying Jensen's inequality and using the concavity of $\log$ in the RHS of \cref{eq:ewamin}, we get:
    \begin{equation*}
        \inner{\cP^*}{c_t} + \frac{1}{\eta}\dkl{\cP^*}{\cP_t} \leq \inner{\cP_t}{c_t}.
    \end{equation*}
    Rearranging terms, we obtain:
    \begin{equation}\label{eq:w0}
        \frac{1}{\eta}\dkl{\cP^*}{\cP_t} 
        \leq \inner{\cP_t-\cP^*}{c_t}.
    \end{equation}
    From \cref{lem:lipschitz}, we have that $c_t$ is $2\gh$-Lipschitz. Hence, $\abs{\inner{\cP_t-\cP^*}{c_t}}\leq 2\gh W(\cP^*, \cP_t)$.
    Therefore, we can rewrite \cref{eq:w0} as follows:
    \begin{equation}\label{eq:w1}
        \frac{1}{\eta}\dkl{\cP^*}{\cP_t} \leq 2\gh W(\cP^*, \cP_t).
    \end{equation}
     Recall that the hypothesis space is bounded by $\rh$ by \cref{assmp:hypothesis}. We can therefore upper bound $W(\cP^*, \cP_t)$ using \cref{lem:pinsker} and \cref{lem:wasserstein_dtv_bound} as:
    \begin{equation}\label{eq:w2}
        W(\cP^*, \cP_t) \leq \rh \sqrt{\frac{\dkl{\cP^*}{\cP_t}}{2}}.
    \end{equation}
    
    Plugging \cref{eq:w2} into \cref{eq:w1}, we have
    \begin{equation}\label{eq:w3}
        \frac{1}{\eta}\dkl{\cP^*}{\cP_t} \leq \gh \rh \sqrt{2 \dkl{\cP^*}{\cP_t}}.
    \end{equation}
     Thus, as KL divergence is non-negative, we have:
    \begin{equation}\label{eq:w4}
        \sqrt{\dkl{\cP^*}{\cP_t}} \leq \sqrt{2} \eta \gh \rh .
    \end{equation}
    Plugging \cref{eq:w4} back into \cref{eq:w2} gives us the desired upper bound for $W(\cP^*, \cP_t)$.
    \begin{equation*}
        W(\cP^*, \cP_t) \leq \eta \gh \rh^2 .
    \end{equation*}
    Now setting $\cP^* = \cP_{t+1}$ implies that $\kappa(t)\leq\eta \gh \rh^2,\,\,\forall t\in[n]$.
\end{proof}
We now instantiate \cref{th:genbarbound1} by picking the EWA algorithm as our online learning strategy in \cref{game}.
\begin{theorem}[Generalization Bound via EWA]\label{th:otbEWA}
    Let ${\{Z_t\}}_{t\in\mathbb{N}}$ be a geometric $\phi$-mixing process with rate $K>0$, and $r>1$. Then for any $P_1\in\Delta_{\cH}$, and any $n>1$, the generalization error $\genbar\left(A,S_n\right)$ of any learning algorithm $A$ trained on $S_n=\left(Z_1,\ldots,Z_n\right)$ drawn from the mixing process ${\{Z_t\}}_{t\in\mathbb{N}}$ is upper bounded by{
    \begin{align*}&\frac{\dkl{P_{A(S_n)}}{P_1}}{n\cdot\eta}+\frac{\eta\cdot B_{\ell}^2 + 4\eta\gh^2\rh^2\log{n}}{2}+\frac{4\gh\rh\log{n}}{n}\\&+\frac{\gz\rz\log{n}}{n}+2\gh\rh\sqrt{\frac{{2\log{n}\log{(\log{n}/\delta)}}
    }{{n}}}+ \frac{B_{\ell}\cdot K}{n},
    \end{align*}}
    with probability at least $1-\delta$, for any $\delta>0$.
\end{theorem}
To prove \cref{th:otbEWA}, we first state the following lemma\footnote{See Appendix A.1 of \citet{Lugosi2023OnlinetoPACCG} for a detailed proof.} to upper bound the regret of the EWA learner.
\begin{lem}[Regret of EWA]\label{lem:ewaregret}
    For any $P_1\in\Delta_\cH$ and any comparator $P^*\in\Delta_\cH$, the regret of EWA satisfies for any learning rate $\eta>0$ is at most
    \begin{equation*}
        \regret_{\text{EWA}}(P^*)\leq\frac{\dkl{P^*}{P_1}||P_1)}{\eta}+\frac{\eta}{2}\sum_{t=1}^n\norm{c_t}_{\infty}^2.
    \end{equation*}
\end{lem}
\begin{proof}[Proof of \cref{th:otbEWA}]
    From \cref{lem:lipschitz} and \cref{lem:ewaregret}, we have for any $P_1\in\Delta_\cH$ and any comparator $P^*\in\Delta_\cH$, the regret of EWA satisfies for any learning rate $\eta>0$ is at most
    \begin{equation}\label{eq:ewa1}
        \regret_{\text{EWA}}(P^*)\leq\frac{\dkl{P^*}{P_1}||P_1)}{\eta}+\frac{n\cdot\eta\cdot B_{\ell}^2}{2}.
    \end{equation}
    Let $\tau=\lceil \log n\rceil-1$. Therefore, from \cref{def:geommix}, we have using $r>1$,
    \begin{equation}\label{eq:ewa2}
        \phi(\tau+1)\leq K\cdot\exp{-(\tau+1)^r}\leq \frac{K}{n}.
    \end{equation}
    Plugging \cref{eq:ewa1} and \cref{eq:ewa2} into \cref{th:genbarbound1} gives us the required bound.
\end{proof}
We note here that \cref{th:otbEWA} gives generalization bounds in terms of the learning rate $\eta$ of the EWA learner. While this is not a problem in and of itself, optimizing the generalization bound would require choosing a learning rate that is data-dependent due to the $\dkl{P_{A(S_n)}}{P_1}$ term, something that is not possible according to \cref{game}. Fortunately, we can obtain data-independent generalization bounds for an appropriate choice of learning rate, as stated in the following corollary of \cref{th:otbEWA}. 
\begin{restatable}{corollary}{etaindependent}\label{cor:etaind}
    Let ${\{Z_t\}}_{t\in\mathbb{N}}$ be a geometric $\phi$-mixing process with rate $K>0$, and $r>1$. Then for any $P_1\in\Delta_{\cH}$, the generalization error $\genbar\left(A,S_n\right)$ of any learning algorithm $A$ trained on $S_n=\left(Z_1,\ldots,Z_n\right)$ drawn from the mixing process ${\{Z_t\}}_{t\in\mathbb{N}}$ is upper bounded by{
    \begin{align*}&\frac{\dkl{P_{A(S_n)}}{P_1}}{\sqrt{n}}+\frac{ B_{\ell}^2 + 4\gh^2\rh^2\log{n}}{2\sqrt{n}}+\frac{4\gh\rh\log{n}}{n}\\&+\frac{\gz\rz\log{n}}{n}+2\gh\rh\sqrt{\frac{{2\log{n}\log{(\log{n}/\delta)}}
    }{{n}}}+ \frac{B_{\ell}\cdot K}{n},
    \end{align*}}
    with probability at least $1-\delta$, for any $\delta>0$.
\end{restatable}
\begin{proof}
    Recall that the R.H.S. in \cref{th:otbEWA} can be upper-bounded by 
\begin{equation}
    \frac{\dkl{P_{A(S_n)}}{P_1}}{n\cdot\eta}+\frac{\eta\cdot B_{\ell}^2}{2}+2\eta\tau\gh^2\rh^2.
\end{equation}
Suppose, for any constant $C>0$, we set the optimal learning rate $$\eta_{\text{opt}}=\sqrt{\frac{C\cdot \dkl{P_{A(S_n)}}{P_1}}{n}}.$$ Then for $\eta<\frac{1}{\sqrt{n}}$, $\dkl{P_{A(S_n)}}{P_1}<\frac{1}{C}$, and the R.H.S. in \cref{th:otbEWA} is at most $O\left(\frac{1}{\sqrt{Cn}}\right)$.

Hence, it is sufficient to analyze the case when $\eta_{\text{opt}}\geq\frac{1}{\sqrt{n}}$ and obtain generalization bounds which are not data-dependent. Plugging $\eta_{\text{opt}}\geq\frac{1}{\sqrt{n}}$ in \cref{th:otbEWA} gives us the desired bounds.
\end{proof}

\section{Discussion and Future Work}\label{sec:conclusion}
In this work, we extend the \otb framework to give generalization bounds for statistical learning algorithms that are trained on data sampled from mixing processes. An immediate avenue of future work is to obtain generalization bounds for non-i.i.d. data in various settings by using different choices of online learners to instantiate our framework as presented in \cref{th:exgenbarbound1} and \cref{th:genbarbound1}.

To compensate for considering the weaker non-i.i.d. assumption on our dataset, our \otb framework requires online learners that are Wasserstein-stable. As noted earlier, the notion of Wasserstein-stability is quite similar to the differential privacy inspired notion of stability proposed by \citet{ambuj2019}, who used their notion of stability to develop regret bounds for a variety of online learning problems, such as follow-the-perturbed-leader algorithms, thereby cementing a connection between differentially private learning algorithms and online learners.  In this light, we ask the following question - can we use the techniques introduced in this paper to analyze generalization error bounds for differentially private learners, especially in the non-i.i.d. setting?

Algorithmic stability and its relation to tight generalization bounds have been recently studied by \citet{gastpar2024fantastic,gastpar2024algorithmstightgeneralizationbounds} in the i.i.d. setting. Due to the dearth of notions of algorithmic stability in the non-i.i.d setting (see \cref{sec:related}), our Wasserstein stability criteria is a potential candidate to extend the results of \citet{gastpar2024fantastic,gastpar2024algorithmstightgeneralizationbounds} to the non-i.i.d. setting.
\subsubsection*{Acknowledgements}
The authors thank the anonymous reviewers of AISTATS 2025 for their useful comments.
\subsubsection*{Note}
The authors of this work are listed in alphabetical order.
\bibliographystyle{unsrtnat}
\bibliography{refs}
\end{document}